\documentclass[sigconf]{acmart}

\AtBeginDocument{%
  \providecommand\BibTeX{{%
    \normalfont B\kern-0.5em{\scshape i\kern-0.25em b}\kern-0.8em\TeX}}}

\acmConference[XXXX '20] {CONFERENCE NAME}{MONTH XX--XX, YEAR}{LOCATION}
\acmBooktitle{BOOK TITLE}
\acmPrice{00.00}
\acmDOI{DOI number}
\acmISBN{ISBN number}

\usepackage[mathscr]{eucal}

\usepackage{stmaryrd}

\usepackage{amsbsy}

\usepackage{bm}

\usepackage{braket}

\usepackage{xparse}

\NewDocumentCommand \RangeSet { G{N} } {[\mathinner{#1}]}

\NewDocumentCommand \LinearSymbol {} {\oast}
\NewDocumentCommand \LeftSymbol {} {\blacktriangleleft}
\NewDocumentCommand \RightSymbol {} {\blacktriangleright}
\NewDocumentCommand \UnfoldSymbol {} {\oslash}

\NewDocumentCommand \SingleSize {s O{J} G{n}} {\IfBooleanTF{#1}{\bar}{} {#2}_{#3}}
\NewDocumentCommand \LinearSize {s O{J}} {\IfBooleanTF{#1}{\bar}{} #2^{\LinearSymbol}}
\NewDocumentCommand \LeftSize {s O{J} G{n}} {\IfBooleanTF{#1}{\bar}{} {#2}_{#3}^{\LeftSymbol}}
\NewDocumentCommand \RightSize {s O{J} G{n}} {\IfBooleanTF{#1}{\bar}{} {#2}_{#3}^{\RightSymbol}}
\NewDocumentCommand \UnfoldSize {s O{J} G{n}} {\IfBooleanTF{#1}{\bar}{} {#2}_{#3}^{\UnfoldSymbol}}
\NewDocumentCommand \FullSize{s O{J} G{N}} {%
  \IfBooleanTF{#1}%
  {\SingleSize*[#2]{0} \times \SingleSize*[#2]{1} \times \cdots \times \SingleSize*[#2]{#3-1}}%
  {\SingleSize[#2]{0} \times \SingleSize[#2]{1} \times \cdots \times \SingleSize[#2]{#3-1}}%
}

\NewDocumentCommand \SizeVec {s O{J}} {\mathbf{\IfBooleanTF{#1}{\bar}{} {#2}}}

\NewDocumentCommand \SingleIndex {s O{j} G{n}} {\IfBooleanTF{#1}{\bar}{} {#2}_{#3}}

\NewDocumentCommand \FullIndex{s O{j} G{N}} {%
  \IfBooleanTF{#1}%
  {(\SingleIndex*[#2]{0}, \SingleIndex*[#2]{1}, \dots, \SingleIndex*[#2]{#3-1})}%
  {(\SingleIndex[#2]{0}, \SingleIndex[#2]{1}, \dots, \SingleIndex[#2]{#3-1})}%
}

\NewDocumentCommand \IndexVec {s O{j}} {\mathbf{\IfBooleanTF{#1}{\bar}{} {#2}}}

\NewDocumentCommand \FullSubscript{s O{j} G{N}} {%
  \IfBooleanTF{#1}%
  {\SingleIndex*[#2]{0} \SingleIndex*[#2]{1} \dots \SingleIndex*[#2]{#3-1}}%
  {\SingleIndex[#2]{0} \SingleIndex[#2]{1} \dots \SingleIndex[#2]{#3-1}}%
}

\NewDocumentCommand \FacMat {s O{U} G{n}} {{\mathbf{\IfBooleanTF{#1}{\bar}{} #2}}^{(#3)}}
\NewDocumentCommand \FacSize {O{I} G{n}} {#1_{#2} \times R_{#2}}

\NewDocumentCommand \Tensor {s O{Y}} {\boldsymbol{\IfBooleanTF{#1}{\bar}{}
{\mathscr{\MakeUppercase{#2}}}}}
\NewDocumentCommand \ColumnBlock {O{\Mz{Y}{n}} G{\ell}} {#1 \left[#2\right]}

\newcommand{\Tra}{{\sf T}}

\newcommand{\V}[2][]{{\bm{#1\mathbf{\MakeLowercase{#2}}}}} 
 
\newcommand{\M}[2][]{{\bm{#1\mathbf{\MakeUppercase{#2}}}}} 
 
\newcommand{\Mz}[3][]{\M[#1]{#2}_{(#3)}}

\newcommand{\lap}{\mathscr{L}}

\usepackage[ruled,vlined]{algorithm2e}
\usepackage{subfigure}

\usepackage{csvsimple}

\newcommand{\vol}{\text{vol}}

\begin{document}
\fancyhead{} 

\title{Hypergraph Random Walks, Laplacians, and Clustering}


\author{Koby Hayashi}
\affiliation{
 \institution{Georgia Institute of Technology}
 \streetaddress{North Ave NW}
 \city{Atlanta}
 \state(GA)
 \postcode{30332}
 \country{USA}}
\email{khayashi9@gatech.edu}

\author{Sinan G. Aksoy}
\affiliation{%
 \institution{Pacific Northwest National Laboratory}
 \streetaddress{902 Battelle Blvd}
 \city{Richland}
 \state(WA)
 \postcode{99354}
 \country{USA}}
\email{sinan.aksoy@pnnl.gov}

\author{Cheong Hee Park}
\affiliation{%
 \institution{Chungnam National University}
 \streetaddress{99 Daehakro}
 \city{Daejeon}
 \postcode{34134}
 \country{Korea}}
\email{cheonghee@cnu.ac.kr}

\author{Haesun Park}
\affiliation{%
 \institution{Georgia Institute of Technology}
 \streetaddress{North Ave NW}
 \city{Atlanta}
 \state(GA)
 \postcode{30332}
 \country{USA}}
\email{hpark@cc.gatech.edu}

\renewcommand{\shortauthors}{Hayashi, Aksoy, Hee Park, and Park}
\begin{abstract}
We propose a flexible framework for clustering hypergraph-structured data based on recently proposed random walks utilizing edge-dependent vertex weights. 
When incorporating edge-dependent vertex weights (EDVW), a weight is associated with each vertex-hyperedge pair, yielding a weighted incidence matrix of the hypergraph. 
Such weightings have been utilized in term-document representations of text data sets. 
We explain how random walks with EDVW serve to construct different hypergraph Laplacian matrices, and then develop a suite of clustering methods that use these incidence matrices and Laplacians for hypergraph clustering. 
Using several 
data sets from real-life applications, we compare the performance of these clustering algorithms experimentally against a variety of existing hypergraph clustering methods. 
We show that the proposed methods produce high-quality clusters and conclude by highlighting avenues for future work. 
\end{abstract}



\maketitle


\section{Introduction}
While graphs serve as a popular tool for modeling a myriad of data analytics tasks, graphs are limited to representing pairwise relationships between objects.  
However, data sets frequently contain multi-way relationships. 
For instance, in a term-document matrix that is frequently used to represent text data, multiple terms are related to each other from their appearance in the same document. 
Multi-way relationships also abound in many contexts, such as when multiple people author a paper, groups of proteins interact, or mutations in multiple genes are associated with a disease. These multi-way relationships are different from multiple binary relationships.
More generally, $k$-way relationships occur whenever information naturally presents as set-valued, bipartite or tabular. 
In such cases, hypergraphs -- generalizations of graphs in which edges may link any number of vertices -- are more appropriate. 

While hypergraph-structured data is widely prevalent, utilizing a hypergraph model to perform analytics tasks is often challenging. 
First and foremost, a primary difficulty concerns how to best {\it represent} a hypergraph for key analytics tasks such as clustering. 
A number of fundamental graph representations such as the adjacency matrix or Laplacian matrix, have no obvious or canonical analog in the hypergraph setting. 
In particular, developing such representations is especially challenging for non-uniform hypergraphs, which appear most often in real applications. 
Furthermore, work by Agarwal \cite{Agarwal:2006:HOL:1143844.1143847} has shown that many hypergraph Laplacian matrices are in fact directly related to various {graph expansions} of a hypergraph, and in this sense, do not fully capture the higher-order relationships modeled by the hypergraph. 
Secondly, another difficulty concerns devising {\it analytic methods} that can effectively utilize these hypergraph representations. 
Indeed, hypergraph representations such as tensors for uniform hypergraphs, where the orders of all hyperedges are same, while faithful in capturing higher-order relationships, may be limited to special cases and difficult or prohibitively expensive to adopt and analyze in practice, due to their large dimensionality or otherwise complicated properties. 

One promising approach for addressing these challenges is rooted in the study of random walks on hypergraphs. 
Much of the work on random walks on hypergraphs has limited applicability to real data because it only considers uniform hypergraphs  \cite{Cooper2013, Lu2013, Liu2018}. 
Other work has considered non-uniform hypergraph random walks, but analyzes simple random walks, in which vertices are chosen uniformly at random from a hyperedge.
However, these random walks have been shown \cite{DBLP:conf/icml/ChitraR19} to be equivalent to a random walk on the graph clique expansion of the hypergraph. 
Recent work by Chitra and Raphael \cite{DBLP:conf/icml/ChitraR19} has shown incorporating so-called edge-dependent vertex weights (EDVW) into the random walk is a necessary condition to circumvent this equivalence, therefore to better capture the higher-order properties of hypergraphs. 
Such vertex weightings, associated with each vertex-hyperedge pair, have appeared in a number of different contexts, such as in term-document matrix represented via tf-idf (term frequency and inverse document frequency), weighting based on the significance in the author order in research paper data, or in general, whenever incidence structures have weighted (rather than binary) cells. 

In this work, we use EDVW random walks to develop a diverse and flexible framework for clustering hypergraph data. 
We explain how to construct several different hypergraph representations in incidence matrices and Laplacians based on EDVW random walks, as well as how one can apply a number of different clustering algorithms to these representations. In addition, we experimentally compare the performance of these EDVW random walk-based clustering approaches to existing hypergraph clustering approaches. 

We organize our work as follows: in Section \ref{sec:prelim}, we provide the necessary preliminaries and briefly review random walks on hypergraphs. 
In Section \ref{sec:lap}, we explain how the probability transition matrix of EDVW random walks may be utilized to construct a number of different hypergraph Laplacians, and survey appropriate possibilities from the literature. 
In Section \ref{sec:clusAlg}, we define clustering methodologies that may be used in conjunction with the aforementioned hypergraph representations. 
In Section \ref{sec:others}, we review other approaches from the literature, consisting of both different hypergraph representations as well as different clustering methodologies. 
In Section \ref{sec:exp} we compare these approaches to ours experimentally: we describe our test datasets, experimental setup, clustering performance evaluation metrics, and report our findings. 

\section{Hypergraphs and random walks}\label{sec:prelim}

\subsection{Preliminaries}

Hypergraphs are generalizations of graphs in which edges can connect any number of vertices. 
More formally, a hypergraph $H=(V,E)$ is a set of vertices $V=\{v_1,\dots,v_n\}$ and a family of edges $E=(e_1,\dots,e_m)$ where $e_i \subseteq V$ for $i=1,\dots,m$. 
A graph is a uniform hypergraph of edge order 2, i.e., every edge $e$ in a graph has $|e|=2$. 
Throughout, we assume the hypergraph has no isolated vertices, i.e. $\bigcup_{e \in E} e = V$, and no empty edges. 
A hypergraph may be represented by its (unweighted) incidence matrix $\M{X} \in \{0,1\}^{|E| \times |V|}$, where 
\begin{align}
\M{X}_{ev}=\begin{cases} 1 & \mbox{ if vertex $v$ belongs to hyperedge $e$,} \\ 0 & \mbox{ otherwise}.  \end{cases} 
\end{align}
(Note that this incidence matrix is sometimes transposed to denote vertex by hyperedge relationships. However, we will use the above notation to make it consistent with the notations in the closely related papers.)
The dual of a hypergraph, denoted $H^*$, is the hypergraph associated with $\M{X}^\Tra$. 

In practice, it is common that hypergraphs are transformed to graphs. 
One such popular transformation replaces each hyperedge with a clique, and is thus called the {\it clique expansion}. 
More precisely, the clique expansion of a hypergraph $H=(V,E)$ is a graph on the same vertex set, with edge set $\{\{u,v\} \, | \, u,v \in e \mbox{ for some } e\in E\}$. 
The clique expansion has weighted adjacency matrix given by $\M{X}^\Tra\M{X}$, where the $(i,j)$ entry denotes the number of shared hyperedges to which vertices $i$ and $j$ belong. 

Although such transformations are convenient in that they facilitate the application of graph-theoretic methods, they also have several drawbacks. 
First, the clique expansion is {\it lossy} in the sense that non-isomorphic hypergraphs may have identical clique expansions. 
In fact, recent work by Kirkland \cite{Kirkland2017} confirms this information loss persists even when hypergraph duality is considered: that is, the {\it pair} of matrices, $\M{X}^\Tra \M{X}$ and $\M{X}\M{X}^\Tra$, corresponding to the weighted clique expansion of a hypergraph and its dual, still does not uniquely identify a hypergraph up to isomorphism. 
This information loss is a primary reason why clique expansion based hypergraph representations are sometimes criticized.

In addition to information loss, another drawback of clique expansions is their density. 
In particular, since each hyperedge of size $k$ contributes ${k \choose 2}$ edges in the clique expansion, a hypergraph with large maximum edge size will have a clique expansion that may be prohibitively dense to analyze or even hold in computer memory. 
Nonetheless, as will soon be clear, clique expansions are a useful reference point for understanding hypergraph random walks, as well their associated Laplacians.

\subsection{Random walks}
\label{Sec_RW}
A random walk on a hypergraph $H=(V,E)$ is a discrete-time Markov chain $X_1,X_2,\dots,$ on state space $V$ defined by given transition probabilities. 
Letting $\omega: E \to \mathbb{R_{+}}$ denote any function that assigns positive weights to the hyperedges of a hypergraph, a standard formulation for a hypergraph random walk may be given as follows: if at time $t$, the current state is $X_t=v_t$, then

\begin{enumerate}
\item Select a hyperedge $e \ni v_t$ with probability proportional to $\omega(e)$.
\item Select a vertex $v \in e$ uniformly at random, and set $X_{t+1}=v$. 
\end{enumerate}

In this random walk, a vertex is chosen uniformly at random from a hyperedge, and we refer to it as a {\it simple random walk}. 
We note that a number of hypergraph random walks studied in the literature follow this form; see \cite{Avin2014, Cooper2013, NIPS2006_3128}. 

For the special case of graphs, describing the random walk as a two-step process is generally redundant: since an edge in a graph can only connect {\it two} vertices, the selection of an incident edge uniquely determines the next state in the chain. 
For more on graph random walks, see \cite{aldous1995reversible}.
In contrast, a hyperedge may connect any number of vertices, any of which could be chosen for the next state. 
Accordingly, the second step above is the key for generalizing random walks on graphs to hypergraphs. 
Focusing on this step, Chitra and Raphael \cite{DBLP:conf/icml/ChitraR19} suggest choosing a vertex from a hyperedge using vertex weightings {\it specific to that hyperedge}. 
More formally, for $e\in E$, letting $\gamma_e: e \to \mathbb{R_+}$ denote the weighting function for a hyperedge $e$, we have if $X_t=v_t$, then 
\begin{enumerate}
\item Select a hyperedge $e \ni v_t$ with probability proportional to $\omega(e)$.
\item Select a vertex $v \in e$ with probability proportional to $\gamma_e(v)$, and set $X_{t+1}=v$. 
\end{enumerate}
The collection of functions $\{\gamma_e\}_{e\in E}$ is called an {\it edge-independent} vertex weighting (EIVW) of $H$ if for every vertex, $\gamma_e(v)=\gamma_{e'}(v)$ for all pairs of hyperedges $e,e'$ containing $v$ \cite{DBLP:conf/icml/ChitraR19}. 
Otherwise, the vertex-weighting is {\it edge-dependent}. 
Edge-dependent vertex weights (EDVW) may also be represented by a weighted incidence matrix $\M{R} \in \mathbb{R}^{|E| \times |V|}_{\geq 0}$,
\begin{align}
\M{R}_{ev} &=  \begin{cases} \gamma_e(v) & \mbox{ if vertex $v$ belongs to hyperedge $e$,} \\ 0 & \mbox{ otherwise}. \end{cases} 
\end{align}
Stated equivalently, the edge-dependent condition is that there is a column in $\M{R}$ in which the nonzero entries are not all equal. 
As we will see in Section \ref{sec:exp}, real data can be naturally represented with EDVW. 
A good example is the term-document matrix for text data sets, where term frequency-inverse document frequency (tf-idf) matrices are natural analogies for EDVW, which we view as a weighted hypergraph incidence matrix (more precisely, as $\M{R}^\Tra$ with terms as vertices and documents as hyperedges). 
EDVW also naturally arise for other types of data, such as author-position for author-paper networks, and association scores for gene-disease data.
In general, whenever vertex-weights are edge-dependent, we will call the random walk described above an {\it EDVW random walk}. 

In a simple random walk on a hypergraph without vertex weights (or with trivial edge-independent vertex weights), a random walk is {\it always} equivalent \footnote{Two random walks on the same state space are equivalent if they have the same transition probability between each pair of states.} to a random walk on the clique expansion graph of the hypergraph, under some edge weighting of the clique expansion as shown in \cite{DBLP:conf/icml/ChitraR19}. 
They show EDVW are necessary for a random walk on a hypergraph not to be characterized simply as a random walk on the clique expansion graph. 
In fact, EDVW random walks may be {\it non-reversible}\footnote{A random walk with probability transition matrix $\M{P}$ and stationary distribution $\V{\pi}$ is reversible if $\V{\pi}_i \M{P}_{ij} = \V{\pi}_{j}\M{P}_{ji}$ for all pairs of states $i,j$. } Markov chains, which implies that they cannot be represented as random walks on {\it any} undirected graph. 
In summary, one potential avenue for capturing higher-order properties of hypergraphs is through EDVW random walks \cite{DBLP:conf/icml/ChitraR19}. 
Next, we describe how to utilize EDVW random walks to construct various Laplacian matrices for hypergraph clustering. 

\section{From random walks to hypergraph Laplacians}\label{sec:lap}
\begin{figure}
\centering
\includegraphics[scale=0.5]{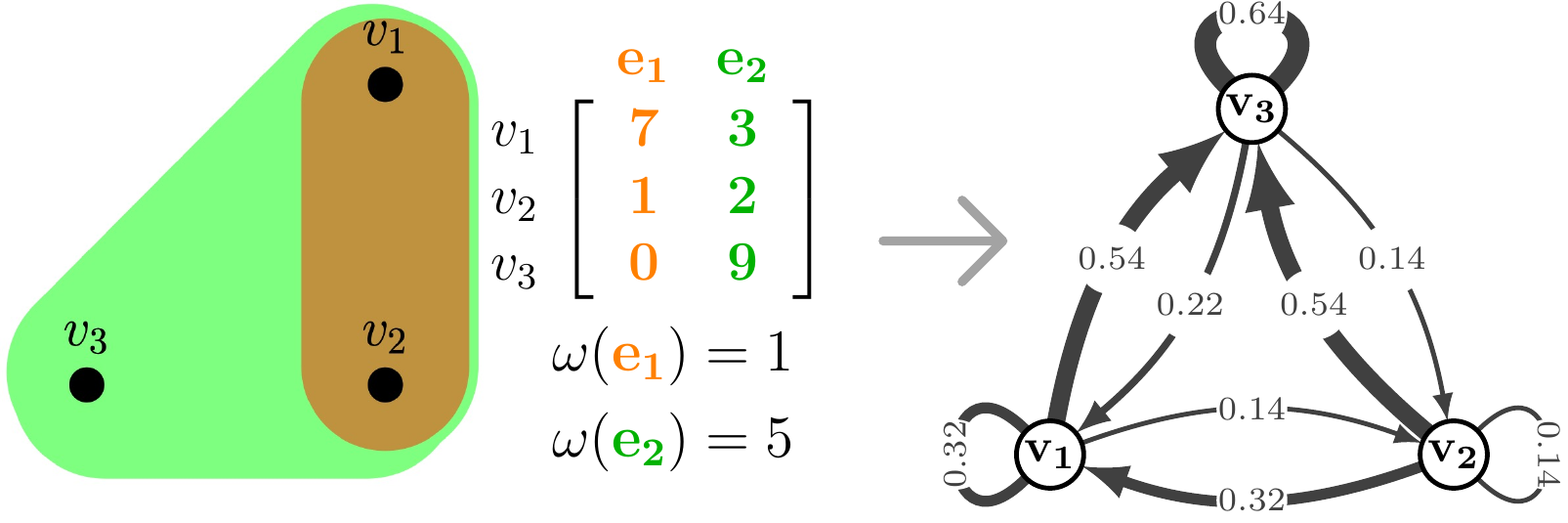}
\caption{A hypergraph with EDVW and hyperedge weights (left) and the representative digraph of the corresponding EDVW random walk (right).}
\label{fig:repDigraph}
\end{figure}
In graph theory, random walks serve as an implicit or explicit foundation for constructing a number of Laplacian matrices. 
For instance, when $\M{P}$ is the transition probability matrix of the random walk, the random walk Laplacian is $\M{I}-\M{P}$. 
The normalized Laplacian matrix popularized by Chung \cite{chung1997spectral} also has eigenvalues that are related to those of $\M{P}$ by elementary shifts and scalings.
Consequently, it is unsurprising the study of Laplacians is deeply intertwined with that of random walks; for more, see the monograph \cite{chung1997spectral}. 

Here, we will explain how one may similarly construct various Laplacian matrices for hypergraphs using the transition probability matrix $\M{P}$ of an EDVW random walk \cite{DBLP:conf/icml/ChitraR19}.
To formally define the transition probability matrix, let $\M{R}$ denote the $|E| \times |V|$ vertex-weight matrix, with $\M{R}_{ev}=\gamma_e(v)$ if $v \in e$ and 0 otherwise. Note that $\M{R}^T$ is a weighted incidence matrix where the weight of each vertex is dependent on the hyperedge it is incident to. 
Similarly, let $\M{W}$ denote the $|V| \times |E|$ hyperedge-weight matrix, with $\M{W}_{ve}=\omega(e)$ if $v\in e$ and 0 otherwise. 
Finally, let $\M{D}_V=\mbox{diag}(\M{W} \V{e})$ and $\M{D}_E=\mbox{diag}(\M{R} \V{e})$ denote the diagonal vertex degree and hyperedge weight matrices, where $\V{e}$ denotes the vector of an appropriate dimension with all its components ones. 
The transition probability matrix for the EDVW random walk is given by
\begin{equation}
    \label{chitra_P}
    \M{P} = \M{D}_V^{-1} \M{W} \M{D}_E ^{-1} \M{R}
\end{equation}

As we explain next, this matrix will be used explicitly to construct a number of different hypergraph Laplacians via its interpretation as an edge-weighted directed graph. 

\subsection{Representative digraph} 
Recall the transition matrix $\M{P}$ completely defines the random walk on a hypergraph $H$. 
One may also represent $\M{P}$, and hence the random walk, as a directed graph (digraph) on vertex set $V$ and edge set $E=\{(i,j) \, | \,  \M{P}_{ij}>0\}$, where the edge weight of $(i,j)$ is simply the transition probability $\M{P}_{ij}$. 
We call this the {\it representative digraph} of the random walk. 
Figure \ref{fig:repDigraph} illustrates an example of a hypergraph with hyperedge weights, EDVW weights presented as $\M{R}^\Tra$, and the representative digraph of the associated random walk. 

When derived from hypergraph random walks, representative digraphs have several notable properties. 
First, they do not contain any source or sink vertices since $\M{P}_{ij}$ is nonzero if and only if $\M{P}_{ji}$ is nonzero as well. 
Furthermore, the representative digraph of a hypergraph random walk is also strongly connected if and only if the hypergraph is connected. 
Consequently, just as any hypergraph may be written as the vertex and hyperedge-disjoint union of connected hypergraphs, its representative digraph may also be represented as the vertex and edge-disjoint union of strongly connected components -- a property that doesn't necessarily hold for directed graphs in general. 
This means one can apply our proposed clustering methodologies to cluster any hypergraph on a per connected component basis, analogous to how graph clustering methodologies are sometimes performed separately on each connected component of the graph. 
Lastly, since representative digraphs always contains loop edges of the form $(i,i)$, this guarantees hypergraph random walks are always {\it aperiodic}. 
Therefore, a random walk on any connected hypergraph is {\it ergodic}, which guarantees convergence to the stationary distribution. 

\subsection{Laplacians based on EDVW random walks}\label{subsec:laps} 
Via the representative digraph, Laplacians for edge-weighted directed graphs naturally serve to construct random-walk based Laplacians for hypergraphs. 
A number of different directed graph Laplacians have been proposed, and they could be utilized in this context. 
Perhaps most natural for our purposes are the combinatorial and normalized digraph Laplacians matrices proposed by Chung \cite{FanChungDLaplacian}. 
Indeed, we note Chitra and Raphael adopt Chung's combinatorial digraph Laplacian as their hypergraph Laplacian in \cite{DBLP:conf/icml/ChitraR19}. 
To define these matrices, recall the stationary distribution $\V{\pi}$ of a random walk is the all-positive dominant left eigenvector of the transition probability matrix $\M{P}$,
\begin{align}
\V{\pi}^\Tra \M{P}=\V{\pi}^\Tra,
\end{align}
scaled to have unit 1-norm. 
By the Perron Frobenius theorem, the stationary distribution $\V{\pi}$ exists if the matrix $\M{P}$ is irreducible, which, in turn, occurs precisely when the representative digraph is strongly connected. 
Letting $\M{\Phi} = \text{diag}(\V{\pi})$, Chung defines the directed combinatorial Laplacian $\M{L}$ and normalized Laplacian $\M{\lap}$ as follows:
\begin{align}
\M{L} &= \M{\Phi} - \frac{\M{\Phi}\M{P} + \M{P}^\Tra \M{\Phi}}{2} \label{eq:dirCombLap} \\
\M{\lap} 
= \M{\Phi}^{-\frac{1}{2}}\M{L} \M{\Phi}^{-\frac{1}{2}} 
& = \M{I}-\frac{\M{\Phi}^{\frac{1}{2}}\M{P}\M{\Phi}^{-\frac{1}{2}} + \M{\Phi}^{-\frac{1}{2}}\M{P}^\Tra\M{\Phi}^{\frac{1}{2}}}{2} \label{eq:dirNormLap}
\end{align}
Clearly, both of the above matrices are symmetric. 
When applied to graphs, $\M{L}$ and $\M{\lap}$ above are equal to the combinatorial and normalized graph Laplacians, respectively. 

Given their explicit basis in random walks and applicability to any irreducible transition matrix $\M{P}$, Chung's Laplacians are natural choices for deriving hypergraph Laplacians based in EDVW random walks. 
Nonetheless, there are other directed Laplacians proposed and studied in the literature that could serve, some of which are asymmetric. 
Bauer \cite{Bauer2012} studies an asymmetric digraph Laplacian that, for digraphs without source vertices, is defined as $\M{I}-\M{P}$. 
Li and Zhang \cite{li2010random} study the asymmetric digraph Laplacian $ \M{\Gamma}=\M{\Phi}^{1/2} \left(\M{I}-\M{P}\right) \M{\Phi}^{-1/2}$, which is related to Chung's normalized Laplacian $\M{\lap}$ above by $\M{\lap}= \tfrac{\M{\Gamma} + \M{\Gamma}^\Tra}{2}$. 
Like Chung's normalized digraph Laplacian $\M{\lap}$, $\M{\Gamma}$ is also a generalization of the normalized graph Laplacian, which it reduces to in the graph case. 

Lastly, one final class of digraph matrices and Laplacians, which have recently received attention in the literature, utilize {\it complex values}. 
Mohar and Guo \cite{Guo2016b} propose a Hermitian digraph adjacency matrix which encodes adjacency using the imaginary unit $i$, or, as proposed more recently \cite{Mohar2020}, the sixth root of unity. 
While such matrices have varied algebraic properties that capture combinatorial properties of the directed graph, their applicability and effectiveness as representations for clustering has yet to be established. 

One exception, however, is recent work by Cucuringu \cite{Cucuringu2019}. 
Utilizing a variant of Mohar and Guo's matrix, Cucuringu proposes a simple digraph spectral clustering algorithm and justifies its effectiveness by analyzing its performance in recovering planted clusters from the Directed Stochastic Block Model, a generalization of the classical SBM \cite{Holland1983}. 
Adapted to our setting, Cucuringu's matrix is
\begin{align}
\M{B}=i \cdot \left(\M{P}-\M{P}^\Tra\right),
\end{align}
where $i$ is the unit imaginary number. 
We note $\M{B}$ is both Hermitian and skew-symmetric. 
This skew-symmetrization of the transition matrix may be thought of as constructing an {\it oriented} digraph (i.e. a digraph without reciprocal edges) from the representative digraph, in which the edge weight between $i$ and $j$ is the difference in their probability transitions, and directionality is encoded by sign. 
For clustering, Cucuringu suggests normalizing this matrix by the diagonal matrix with $\M{S}_{ii}=\sum_{j} |\M{B}_{ij}|$, i.e., forming $\M{S}^{-1}\M{B}$.

\subsection{Relationship with the clique expansion}
Lastly we clarify how, when applied to the representative digraph of the EDVW hypergraph random walk, the above matrices are related to Laplacians of the hypergraph's clique expansion graph. 
This question was considered by Agarwal \cite{Agarwal:2006:HOL:1143844.1143847}, who showed a number of other hypergraph Laplacians are equivalent to the graph Laplacians of the clique expansion. 

As we've noted, an EDVW hypergraph random walk may be non-reversible, in which case there is no graph (including the clique expansion) with probability transition matrix equal to that of the hypergraph, as all graph random walks are necessarily reversible \cite{aldous1995reversible}. 
Consequently, it immediately follows that the hypergraph Laplacian matrix $\M{I}-\M{P}$ cannot be characterized as $\M{I}-\M{Q}$, where $\M{Q}$ is the probability transition matrix of a random walk on a graph. 
Furthermore, we prove an analogous statement holds for Li and Zhang's digraph Laplacian matrix, $\M{\Phi}^{1/2} \left(\M{I}-\M{P}\right) \M{\Phi}^{-1/2}$.

\begin{proposition} \label{prop:liZhang}
Let $\M{P}$ denote the probability transition matrix of an EDVW random walk on a connected hypergraph. If the Markov chain given by $\M{P}$ is non-reversible, there does not exist any edge-weighted graph $G$ such that
\begin{align}
\M{\Phi}^{1/2} \left(\M{I}-\M{P}\right) \M{\Phi}^{-1/2} &= \M{\Pi}^{1/2} \left(\M{I}-\M{Q}\right) \M{\Pi}^{-1/2},
\end{align}
where $\M{Q}$ denotes the probability transition matrix of a random walk on $G$, and $\M{\Phi}, \M{\Pi}$ denote diagonal matrices with the stationary distributions of $\M{P},\M{Q}$, respectively, on the diagonal.
\end{proposition}
\begin{proof}
Assume such a graph exists. Then
\begin{align}
\M{P}_{ij}=\left(\sqrt{{\phi_j \pi_i}}/\sqrt{{\phi_i \pi_j}}\right)\M{Q}_{ij}, \label{eq:liClique}
\end{align}
where $\phi_i = \M{\Phi}_{ii}$ and $\pi_i = \M{\Pi}_{ii}$. Since the random walk given by $\M{P}$ is ergodic, Eqn. $(\ref{eq:liClique})$ implies random walk given by $\M{Q}$ is ergodic as well. Furthermore, since all random walks on graphs are time reversible, the random walk given by $\M{Q}$ is time reversible. Ergodic, time-reversible random walks satisfy Kolmogorov's criterion; applied to $\M{Q}$, Kolmogorov's criterion states that for any set of vertices, $\{v_1,\dots,v_n\}$, we have\\
$ \M{Q}_{v_1v_2}\M{Q}_{v_2 v_3}\dots \M{Q}_{v_{n-1} v_n}\M{Q}_{v_n v_1} =\M{Q}_{v_1v_n}\M{Q}_{v_{n} v_{n-1}} \dots \M{Q}_{v_3 v_2}\M{Q}_{v_2 v_1}. $\\
Now, 
\begin{align*}
\M{P}_{v_1v_2}\M{P}_{v_2v_3}\dots \M{P}_{v_n v_1}
&= \M{Q}_{v_1v_2}\M{Q}_{v_2v_3}\dots \M{Q}_{v_nv_1}\\ 
&= \M{Q}_{v_1v_n}\M{Q}_{v_{n} v_{n-1}}\dots \M{Q}_{v_2 v_1} 
= \M{P}_{v_1v_n}\M{P}_{v_{n}v_{n-1}}\dots \M{P}_{v_2 v_1}
\end{align*}
due to Eqn. $(\ref{eq:liClique})$, which implies $\M{P}$ is the transition matrix of a time-reversible Markov chain, a contradiction. 
\end{proof}
The above result means the EDVW random walked based hypergraph Laplacians $\M{I}-\M{P}$ and  $\M{\Phi}^{1/2} \left(\M{I}-\M{P}\right) \M{\Phi}^{-1/2}$ cannot be characterized as Laplacians of graphs, and hence are not subject to Agarwal's criticism of hypergraph Laplacians in \cite{Agarwal:2006:HOL:1143844.1143847}. 
However, Chung's digraph Laplacians symmetrize the probability transition matrix and, as has been previously noted in \cite{li2010random, gleichdcluster}, {\it can} be understood as Laplacians of weighted undirected graphs. 
In our context, this observation may be stated more precisely as follows:

\begin{proposition} \label{prop:chungEq}
Let $\M{L}$ and $\M{\lap}$ denote the hypergraph Laplacians obtained from applying Chung's digraph Laplacians defined in Eqns $(\ref{eq:dirCombLap})$ and  $(\ref{eq:dirNormLap})$, to the representative digraph of a random walk on a connected hypergraph $H$. The matrix
\begin{align}
\M{A}=\frac{\M{\Phi}\M{P} + \M{P}^\Tra \M{\Phi}}{2} \label{eq:chung_undir}
\end{align}
is the adjacency matrix of the clique expansion graph of $H$, under an edge-weighting. Furthermore, the edge-weighted graph given by $\M{A}$ has combinatorial and normalized graph Laplacian matrices equal to $\M{L}$ and $\M{\lap}$, respectively. 
\end{proposition}
\begin{proof}
To see $\M{A}$ is the adjcency matrix of the clique expansion graph of $H$, observe that, by definition of an EDVW random walk on a hypergraph, $\M{P}_{ij}$ (and hence $\M{A}_{ij}$) is nonzero if and only if vertices $i,j \in e$, for some edge $e$ in the hypergraph $H$. 
This is precisely the edge condition in the clique expansion definition.
The combinatorial and normalized Laplacians of $\M{A}$ are the same as $\M{L}$ and $\M{\lap}$, respectively, since
$\M{A}\V{e} =\frac{1}{2}(\M{\Phi}(\M{P}\V{e}) + \M{P}^\Tra(\M{\Phi}\V{e}))
=\frac{1}{2}(\M{\Phi} \V{e} + \M{P}^\Tra \V{\pi})
= \V{\pi}$
and hence the weighted graph described by $\M{A}$ has diagonal degree matrix $\M{D}=\M{\Phi}$. Substituting $\M{D}$ and $\M{A}$ into $\M{D}-\M{A}$ and $\M{I}-\M{D}^{-1/2}\M{A}\M{D}^{-1/2}$ yields the result. 
\end{proof}

In this sense, Agarwal's criticism in \cite{Agarwal:2006:HOL:1143844.1143847} also applies to hypergraph Laplacians derived from Chung's digraph Laplacians.
However, as we will highlight further in the next section, Chung's digraph Laplacian matrices still preserve key information about the random walk important for clustering and hence serve as effective representations for our approach. 

Lastly, we note Proposition \ref{prop:chungEq} also answers a question of Chitra \cite[Section 5.1]{DBLP:conf/icml/ChitraR19} on whether there exist edge weights on the clique expansion such that its combinatorial Laplacian is ``close" to the hypergraph Laplacian obtained from Chung's digraph combinatorial Laplacian. 
Indeed, the proposition shows there exists edge weights such that {\it equality} holds. 
Nonetheless, it remains to be seen whether a more explicit formula for these edge weights may be obtained solely in terms of the hypergraph's hyperedge and vertex weights, rather than invoking the stationary distribution, as in Eqn. (\ref{eq:chung_undir}). 

\section{Proposed Clustering Methods} \label{sec:clusAlg}

Given a hypergraph $H=(V,E)$ with EDVW, and a desired number of clusters $k$, our goal is to partition $V$ into disjoint subsets $S_1,\dots, S_k$, such that a cluster quality objective function is optimized. 
Recall that in {\it graph} clustering, one such well known function is the normalized cut (Ncut), which measures the weight between each cluster $S$ to its complement $S^c$ relative to the cluster ``volume". 
More precisely
\begin{equation}
    \label{Ncut_def}
	\text{Ncut}(S_1, \dots, S_k) = \frac{1}{2}\sum_{i=1}^k\frac{\vol(\partial S_i)}{\vol(S_i)},
\end{equation}
where $\vol(S_i)=\sum_{u \in S_i} (\M{A}\V{e})_u$, $\vol(\partial S_i)=\sum_{u \in S_i, v \in S_i^c}\M{A}_{uv}$ and $\M{A}$ is the adjacency matrix of the graph. 
However, in our case we are utilizing the representative digraph of the EDVW hypergraph random walk (and its associated Laplacians) as our basis for clustering. 
Thus, we require a notion of Ncut for directed graphs, which, in turn, requires directed notions of $\vol(S)$ and $\vol(\partial S)$. 
Chung proposed such digraph analogs of volume, which are based in the probability transition matrix  $\M{P}$ and stationary distribution $\V{\pi}$:
\begin{align}
\vol(S_i) &= \sum_{u \in S_i}\V{\pi}_u, \label{dvol} \\
\vol(\partial S_i) &= \sum_{u \in S_i, v \in S_i^c} \V{\pi}_u \M{P}_{uv}. \label{dboun}
\end{align}
As shown in \cite{gleichdcluster}, we note $\vol(\partial S)=\vol(\partial S^c)$ and that the function $F_{\V{\pi}}=\V{\pi}_u \M{P}_{uv}$ is an example of a {\it circulation} function, a general type of flow on the directed graph (see \cite{FanChungDLaplacian} for more). 

We also note these directed notions of volume yield an elegant and intuitive random walk interpretation of Ncut, which was observed in the graph case by Meila and Shi \cite{Meila01arandom}. 
In particular, if we let $\Pr(S^c | S)$ denote the probability of transitioning to a vertex in $S^c$ given the current state is a vertex in $S$, then it is straightforward to show  $\Pr(S^c | S)= \frac{\text{vol}(\partial S)}{\vol(S)}$. 
By definition of Ncut, we thus have:
\begin{align}
\text{Ncut}(S,S^c) = \Pr(S^c | S)+ \Pr(S | S^c).
\end{align}

Now, returning our attention to Chung's digraph Laplacian matrices, recall Chung's digraph Laplacians are equivalent to the graph Laplacians associated with a particular edge-weighted graph derived from the digraph -- that is, the graph with edge-weighted adjacency matrix defined in Eqn. (\ref{eq:chung_undir}). 
Further, Gleich showed $\vol(S)$ and $\vol(\partial S)$ of this graph are equal to their directed analogs (defined above) of the associated digraph \cite[p.~7]{gleichdcluster}. 

This means that any graph clustering algorithm which minimizes graph Ncut will, when applied to Chung's digraph Laplacians, minimize the directed analog of Ncut obtained by using the directed volume definitions in Eqns. $(\ref{dvol})-(\ref{dboun})$ in the Ncut definition in Eqn. $(\ref{Ncut_def})$. 
We test two algorithms to obtain clusterings on Chung's Laplacian. 
The spectral method is motivated by Zhou et al.'s \cite{zhou2005learning} algorithm which heuristically minimizes the Ncut on Chung's Laplacian and another method proposed by Ng, Jordan, and Weiss \cite{andrew_Ng_spec_clust} which also heuristically minimizes the Ncut.
Applying a similar algorithm to Chung's normalized digraph Laplacian yields our suggested hypergraph spectral clustering method, Algorithm \ref{RDC_spec}.
Additionally, we propose using a Symmetric Non-negative Matrix factorization based algorithm for graph clustering based on the framework from Kuang et al. \cite{Kuang_symmetricnonnegative,SymNMF_Jogo}. 

\begin{algorithm} [tb]
\SetAlgoLined
\KwIn{A connected hypergraph $H=(V,E)$ with hyperedge weights and edge-dependent vertex weights, and desired number of clusters $k\geq 2$.}
 \nl Construct $\M{P}$ as in Eqn. (\ref{chitra_P}). \\
 \nl Construct $\M{\Phi} = \text{diag}(\V{\pi})$ such that $\V{\pi}^\Tra\M{P} = \V{\pi}$ and $\|\V{\pi}\|_1 = 1, \V{\pi} > 0$ \\
 \nl Construct $\M{T}$ as in Eqn. (\ref{eq:laplacian_core}). \\
 \nl Compute the $k$ eigenvectors paired with the $k$ largest eigenvalues of $\M{T}$ and collect them into the columns of the matrix $\M{U} \in \mathbb{R}^{|V| \times k}$.\\
 \nl Normalize the rows of $\M{U}$ with respect to the $L_2$ norm.\\
 \nl Cluster the rows of the matrix $\M{U}$ using $k$-means.\\
 \KwOut{$k$ vertex clusters}
 \caption{Representative Digraph Clustering-Spec (RDC-Spec)} 
 \label{RDC_spec}
\end{algorithm}

\begin{algorithm}[tb]
\SetAlgoLined
\KwIn{A connected hypergraph $H=(V,E)$ with hyperedge weights and edge-dependent vertex weights, and desired number of clusters $k\geq 2$.}
 \nl Construct $\M{P}$ as in Eqn. (\ref{chitra_P}). \\
 \nl Construct $\M{\Phi} = \text{diag}(\V{\pi})$ such that $\V{\pi}^\Tra\M{P} = \V{\pi}$ and $\|\V{\pi}\|_1 = 1, \V{\pi} > 0$ \\ \nl Construct $\M{T}$ as in Eqn. (\ref{eq:laplacian_core}). \\
 \nl Compute a rank $k$ Symmetric NMF of $\M{T}$ and collect the factors into the columns of the matrix $\M{U} \in \mathbb{R}_{\ge 0}^{|V| \times k}$.\\
 \nl Assign each row of $\M{U}$ to the column index of its max element. Use these as cluster assignments.\\
 \KwOut{$k$ vertex clusters}
 \caption{Representative Digraph Clustering-SymNMF (RDC-Sym)} 
 \label{RDC_sym} 
\end{algorithm}

When the hypergraph is disconnected, Algorithms \ref{RDC_spec} and \ref{RDC_sym} may be applied per connected component. 
We also note this algorithm utilizes the matrix $\M{T}$,
\begin{equation}
    \label{eq:laplacian_core}
    \M{T} = \M{I} - \M{\lap} = \frac{\M{\Phi}^{\frac{1}{2}}\M{P}\M{\Phi}^{-\frac{1}{2}} + \M{\Phi}^{-\frac{1}{2}}\M{P}^\Tra\M{\Phi}^{\frac{1}{2}}}{2}
\end{equation}
which is slightly different from Chung's digraph Laplacian, which is $\M{\lap}=\M{I}-\M{T}$. 
This modification is made to ensure the input is non-negative. 
This assumption is not necessary for spectral clustering, but necessary for non-negative matrix factorizations. 

Our framework can be applied to multi modal data analysis frameworks that seek to utilize hypergraph information.
For example Joint-NMF (JNMF), as proposed by Du et al. \cite{Du_JNMF}, is able to utilize multiple sources of information to perform clustering.
This is further discussed in Section \ref{sec:USP}.

Symmetric Non-negative Matrix Factorization (SymNMF) solves
\begin{equation}
    \label{SymNMF_obj}
    \min_{\M{F} \ge 0}\| \M{S} - \M{F}\M{F}^\Tra \|_F^2
\end{equation}
where $\M{S}$ is a symmetric, non-negative matrix and $\M{F} \in \mathbb{R}_{\ge 0}^{|V| \times r}$, $r$ is some positive integer which is usually set to the number of clusters when clustering.
It has been shown that SymNMF can achieve state-of-art results on various graph clustering tasks such as image segmentation \cite{SymNMF_Jogo,Kuang_symmetricnonnegative}.
Additionally, Kuang et al. \cite{SymNMF_Jogo} show SymNMF and Spectral clustering minimize the same objective function but with different constraints.
While SymNMF aims to solve Eqn. (\ref{SymNMF_obj}), spectral clustering aims to solve the same objective but imposes that $\M{F}^\Tra\M{F} = \M{I}$ instead of $\M{F}\ge0$.
While it depends on the algorithms utilized, SymNMF and Spectral Clustering have comparable complexities of (very) roughly $O(|V|^2k)$ per iteration.
SymNMF is dominated by the formation of the normal equations which requires the matrix products $\M{F}^T\M{F}$ and $\M{F}^T\M{S}$, whereas
Spectral Clustering is dominated by the cost of computing a truncated spectral decomposition of $\M{S}$.
We adapt SymNMF clustering to our framework in Algorithm \ref{RDC_sym}.

Lastly, we conclude this section by acknowledging other approaches that may be taken to cluster hypergraphs via the representative digraph of an EDVW hypergraph random walk. 
In particular, the aforementioned work by Cucuringu \cite{Cucuringu2019} takes a Stochastic Block Model (SBM) approach towards clustering digraphs. 
The SBMs are probabilistic models that generate random networks with planted communities; for more on SBMs and clustering, see \cite{Lei2015,abbe2017community}. 
Cucuringu shows that, under mild assumptions on the parameters of the Directed Stochastic Block Model, the number of vertices misclassified by their algorithm is well-bounded, with high probability. 
This is shown by applying particular tools from random matrix theory, which critically rely upon both the Hermitian and skew-symmetry properties of the matrix. 
The complex-values in Cucuringu's input representation serve the purpose of allowing digraph edge-directionality to be encoded in a matrix with these properties. 
For details interested readers may refer to \cite{Cucuringu2019}.
Although not explored further in this work, we note their algorithm may be applied to any edge-weighted digraph, and thus could be applied to the representative digraph of an EDVW hypergraph. 

\section{Existing Hypergraph Clustering Methods} \label{sec:others}

We've outlined an EDVW random-walk based framework for clustering hypergraphs that offers flexibility both in the choice of representation, as well as clustering method. 
Now, we will survey other clustering approaches that utilize different hypergraph representations or clustering methods than what we have proposed. 
Then, in Section \ref{sec:exp}, we will compare our clustering framework against these methods on text-document and other datasets. 
Before describing the details, it is helpful to take a broader viewpoint of hypergraph clustering approaches and briefly discuss how EDVW-based methods fit within this literature. 

Much of the recent work on hypergraph clustering is fundamentally centered around the question of how hyperedges can be cut or alternatively how a vertex contributes to a hyperedge.
In the context of the graph expansion-approaches, in which the aforementioned clique expansion is studied in place of the hypergraph, this question is answered by how the edges in the expansion graph are assigned weights.
These weights are usually uniform with respect to a single hyperedge.
Consequently, each vertex within a hyperedge is treated equally; for example, Zhou's Laplacian and hypergraph clustering formulation \cite{NIPS2006_3128} is one such example.

Instead of clustering based on how hyperedges are cut, an alternative approach is to consider how network motifs (certain small subgraphs, such as a 3-clique) are cut. 
This approach was suggested by Benson \cite{DBLP:journals/corr/BensonGL16a} and further explored in \cite{li2017inhomogeneous}.
In the motif-based clustering algorithms proposed in \cite{DBLP:journals/corr/BensonGL16a} no matter how a motif is cut it incurs a constant penalty. 
Later, the authors of \cite{li2017inhomogeneous} discussed the idea of inhomogeneous hypergraphs which can be thought of as motif clustering where different cuts in a motif incur different penalties.
The cost of some of the different cuts are assumed to be given a priori: for example, if a hyperedge contains vertices $a,b,c,d$, then a weight for the cost of separating $a,b$ and $c,d$ is given.
This information is then used to form each hyperedge into a clique that preserves cut constraints, yielding a clique expansion style method.

EDVW hypergraphs seek to address the same problem but from a different perspective.
Instead of having a-priori penalties for edges or relationships between vertices within a hyperedge, for each hyperedge, EDVW give us data-driven values for how much each vertex individually contributes to that hyperedge. 
This information is then used to describe a random walk, which serves as the basis for deriving representations utilized by our proposed clustering algorithms. Below, we survey other hypergraph clustering methods. 

\begin{itemize}
\item Clique-expanded Hypergraph Clustering (CHC): Proposed by Zhou in  \cite{NIPS2006_3128}, this algorithm expands each hyperedge in a given hypergraph into a clique and assigns a uniform weight value to each edge formed. Spectral clustering is then run on the resulting weighted graph which has the Laplacian 
\begin{align}
    \label{Zhou_lap}
    \M{\Delta} = \M{D}_V^{-1/2}\M{X}^\Tra\M{Z}\M{D}_E^{-1}\M{X} \M{D}_V^{-1/2}
\end{align}
where $\M{X} \in \{0,1\}^{|E|\times |V|}$ is the incidence matrix, $\M{Z} \in \mathbb{R}_+^{|E|\times |E|}$ is the diagonal matrix containing the weights of hyperedges, $\M{D}_E = \text{diag}(\M{X}\V{e})$ and $\M{D}_V$ is a diagonal matrix where the $(i,i)$th entry is $\V{e}^\Tra \M{Z}\V{x}_i$ and $\V{x}_i$ is the $i$th row of $\M{X}$.
\item NMF for Text-Clustering (NMF) \cite{nfm_txt_clst} solves the problem  $$\min_{(\M{U},\M{M}) \ge 0} \| \M{R}^\Tra - \M{U}\M{M}^\Tra \|_F^2,$$ $\M{U} \in \mathbb{R}_{\ge0}^{|V| \times k}$ and  $\M{M} \in \mathbb{R}_{\ge0}^{|E| \times k}$, then column normalizes the document factor matrix to unit 2-norm and uses its max row indices to assign documents to clusters.
\item K-Means (KM) runs the K-Means clustering algorithm to obtain a document clustering on the tf-idf matrix.
\item Clique Random Walk Clustering (CRWC): In \cite{NIPS2006_3128} Zhou proposed the uniform random walk discussed in Section \ref{Sec_RW}. The probability transition matrix for this random walk is $\M{P} = \M{D}^{-1}_V \M{X}^\Tra \M{Z} \M{D}^{-1}_E \M{X}$. This matrix is fed into Algorithm \ref{RDC_spec} in place of the EDVW stochastic matrix on line 1. The matrices in this equation are the same as for CHC above. This method is included to assess the value of using EDVW vs EIVW.
\item Spectral Bi-Clustering (SBC) \cite{specBiClust} uses a weighted incidence matrix to obtain both edge and vertex clusterings. Following \cite{specBiClust}, we use the tf-idf matrix to cluster documents. An SVD is applied to a normalized tf-idf matrix $\M{D}_1^{-1/2} \M{R} \M{D}_2^{-1/2}$ where $\M{D}_1 = \text{diag}(\M{R}\V{e})$ and $\M{D}_2 = \text{diag}(\M{R}^\Tra\V{e})$. Then k-means is run on a set of truncated-normalized singular vectors.
\end{itemize}

Table \ref{tab:alg_table} summarises various attributes in hypergraph clustering algorithms and indicates if each algorithm utilizes them.
The attributes are 1) Random Walk, if an algorithm is based on a random walk formulation, 2) Spectral, if an algorithms uses the spectrum of a matrix to cluster, 3) EDVW, if an algorithm uses information from edge-dependant vertex weights, and 4) EIVW if an algorithm uses edge-independent vertex weights.
A $\checkmark$ indicates a ``has" or yes while a $\times$ indicates a ``has not" or no.

\begin{table}[tb]
\begin{center}
 \begin{tabular}{||c c c c c||} 
 \hline
 Alg & Random Walk & Spectral & EDVW & EIVW \\ [0.5ex] 
 \hline\hline
 RDC-Spec & $\checkmark$ & $\checkmark$ & $\checkmark$ & $\times$ \\
 \hline
 RDC-Sym & $\checkmark$ & $\times$ & $\checkmark$ & $\times$ \\
 \hline
 CHC & $\checkmark$ & $\checkmark$ & $\times$ & $\checkmark$ \\
 \hline
 NMF & $\times$ & $\times$ & $\checkmark$ & $\times$ \\
 \hline
 KM & $\times$ & $\checkmark$ & $\checkmark$ & $\times$ \\
 \hline
 CRWC & $\checkmark$ & $\checkmark$ & $\times$ & $\checkmark$ \\
 \hline
 SBC & $\times$ & $\checkmark$ & $\checkmark$ & $\times$ \\
 \hline
 \hline
\end{tabular}
\end{center}
\caption{Algorithm characteristics on whether random walks are used (Random Walk), eigenvalues are used to cluster (Spectral), and whether edge-dependent or edge-independent vertex weights are considered (EDVW, EIVW).}
\label{tab:alg_table}
\end{table}

\section{Experiments} \label{sec:exp}
We test our proposed methods on a number of datasets. 
Most of these data sets come with ground truth allowing various metrics to be used to assess the output quality of an algorithm.

\subsection{Data Sets and Preprocessing}
\label{DSandProc}
We experiment on the following four data sets:
\begin{itemize}
    \item 20-Newsgroups
    \item United States Patents Data
    \item Reuter's Corpus Volume 1
    \item Gene-Disease Data
\end{itemize}
For each text data set the documents are taken as vertices and the words are taken as hyperedges.
The corresponding EDVWs are the tf-idf values.
That is for the matrix $\M{R}$, as in Eqn. (\ref{chitra_P}), rows correspond to words and columns corresponding to documents.
The entry in the $w$th row and $d$th column of $\M{R}$ is the tf-idf value between the $w$th word and the $d$th document.
Each hyperedge weight is computed as the standard deviation of a word (row of $\M{R}$) in the tf-idf matrix and encoded in the matrix $\M{W}$ as in Eqn. (\ref{chitra_P})\cite{DBLP:conf/icml/ChitraR19}. 
A few preprocessing steps are applied to each dataset.
Following \cite{specBiClust}, words that appear in over a certain fraction of the data-sets are removed, this fraction is determined experimentally and the best one is chosen for each algorithm.
We refer to this parameter as the sparsity parameter.
Similarly, words that only appear in a single document or that appear in no documents are removed.
Every data set is checked to make sure it consists of a single connected component.
Various information about each data set is listed in Table \ref{text_data_stats_table}.

\begin{table}
\begin{center} 
 \begin{tabular}{||c c c c c||} 
 \hline
 Data & \#-Vertices & \#-Hyperedges & NNZ (Ind. Mtrx) & \#-Clstrs \\ [0.5ex] 
 \hline\hline
 G1 & 1498 & 22755 & .0049  & 4 \\
 \hline
 G2 & 1545 & 19081 & .0048 & 4 \\
 \hline
 G3 & 1430 & 19412 & .0048 & 4 \\
 \hline
 G4 & 1945 & 20260 & .0039 & 5 \\
 \hline \hline
 A22 & 835 & 4496 & .0160 & 15 \\
 \hline
 A42 & 965 & 4692 & .0140 & 17 \\
 \hline
 D02 & 744 & 4499 & .0151 & 12 \\
 \hline
 B06 & 688 & 4106 & .0164 & 12 \\
 \hline \hline
 RCV1 & 9625 & 29969 & .0023 & 4 \\
 \hline \hline
 GD & 2261 & 12368 & .0041 & na \\
 \hline
 \hline
\end{tabular}
\end{center}
\caption{
Hypergraph data statistics. For text data, statistics are reported after pruning with a sparsity parameter of 0.2. The labels G1-G4, A22-B06, RCV1, and GD refer to 20Newsgroups, US Patents, Reuters Corpus Volume 1, and gene-disease datasets, respectively; see Section \ref{subsec:results} for more. 
}
\label{text_data_stats_table}
\end{table}

\subsection{Clustering Metrics}
In order to assess clustering quality on the 3 text data sets we compute the Normalized Mutual Information (NMI), average F1 score, and Jaccard index.
These metrics are based on comparing two clustering results represented by $X$ and $Y$, which may be treated as two vectors of integer labels where $X_i = X_j$ only if the $i$th and $j$th ($i \ne j$) vertices are assigned to the same cluster.
For our purposes one of $X$ and $Y$ is the ground truth cluster labels and the other is the predicted cluster labels from some clustering algorithm.
Normalized mutual information (NMI) is given by $\frac{2\cdot I(X,Y)}{H(X) + H(Y)}$, where $I(X,Y)$ is the mutual information and $H(X)$ is the entropy.
The average F1 score is a generalization of the F1 score for multi-clusterings. $|X| \cup |Y|$ is the number of points the two assignments classify the same and $|X| \cap |Y|$ is $|X| + |Y| - |X \cup Y|$. 
For two clusters, the F1 score (see \cite{F1_Leskovec}) is $F_1(X,Y) = 2 \frac{|X \cap Y|}{|X| + |Y|}$, and the Jaccard index is $J(X,Y) = \frac{|X \cap Y|}{|X \cup Y|}$.
In order to handle multi-clustering, we use the Khun-Munkres algorithm to compute a matching between clusters to maximize the overall score for both the Jaccard and average F1 scores.
This approach is similar to that used by Kuang et al. \cite{Kuang_symmetricnonnegative}.

\subsection{Results} \label{subsec:results}

\subsubsection{20-Newsgroups}
\label{sec:20news}
\begin{figure}[tb] 
\centering
\includegraphics[width=0.5\textwidth]{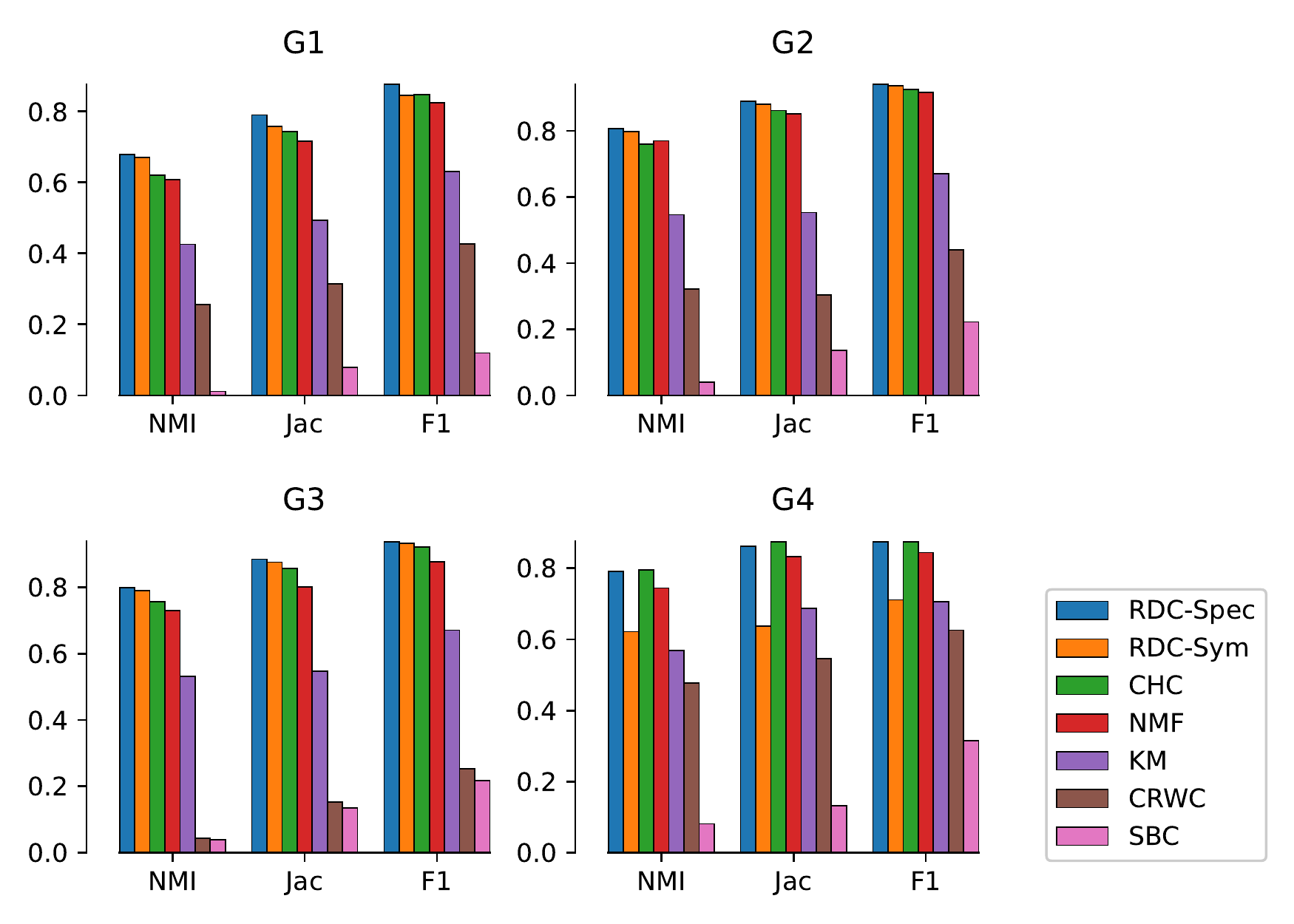}
\caption{Clustering scores based on NMI, Jaccard index, and F1 score, for 4 subsets of the 20News data set.}
\label{20news_plts}
\end{figure}

The first set of experiments was done on the 20-News\footnote{http://qwone.com/~jason/20Newsgroups/} data set.
This data set consists of 4 major categories each with a varying number of subcategories.
We select 4 subsets of the subcategories to cluster on.
The subcategories for each experiment are G1) OS Microsoft Windows, automobiles, cryptography, and politics-guns, G2) atheism, computer graphics, medicine, and Christianity  G3) Windows X, motorcycles, space, religion-miscellaneous and G4)
computer graphics, OS Microsoft Windows, IBM PC hardware, MAC hardware, and Windows X.
The first three groups are expected to be well separated while the last is expected to present a more challenging clustering problem.
All clusters have between 318 and 398 documents.
Some basic statistics from the 20-News data set subsets we use are given in Table \ref{text_data_stats_table} for G1-G4.
These hypergraph statistics are all from using a pruning parameter of 0.2 which gives representative results for all algorithms.
From Figure \ref{20news_plts} we observe that our framework gives competitive results for all subsets of the 20-News data.
\subsubsection{US-Patents}
\label{sec:USP}
\begin{figure}[tb]
\centering
\includegraphics[width=0.52\textwidth]{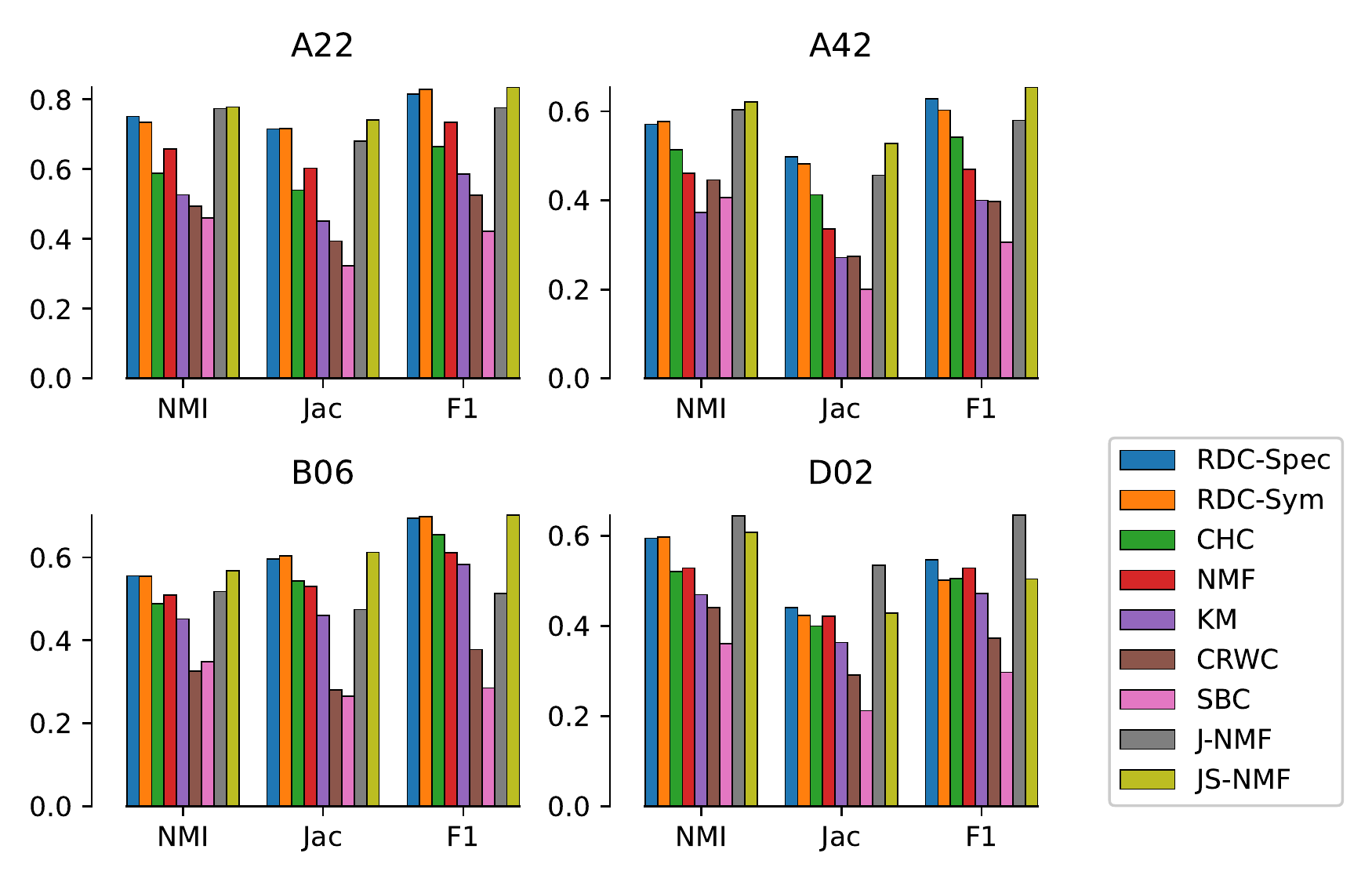}
\caption{Clustering scores based on NMI, Jaccard index, and F1 score, for four
US Patents data sets.}
\label{USP_plts}
\end{figure}

This data was originally processed by Du et al. \cite{Du_JNMF} from the Patents View\footnote{www.patentsview.org} website and contains word count and citation information for a number of patents claims for 13 different categories (denoted A22, B06, etc.) each with multiple sub-classes.
These sub-classes are used as ground truth and only patents belonging to a single sub-class are used.
Finally, only sub-classes with 40 or more patents are kept.
For our experiments we selected the categories A22, A42, D02, and B06 as these each have a sufficient number of patents that belong to only a single sub-class.

As this data also contains information on which patents reference each other, this additional citation information may be utilized via Joint-NMF (J-NMF) style algorithms, which solve
\begin{equation}
    \label{eq:JNMF}
   \min_{ \{\M{M},\M{Z},\M{\Tilde{M}} \} \ge 0} \| \M{X} - \M{Z}\M{M}^\Tra \|_F^2 + \gamma\| \M{S} - \M{M} \M{\Tilde{M}}^\Tra \\|_F^2 + \beta \|\M{M} - \M{\Tilde{M}}\|_F^2
\end{equation}
where $\M{X} \in \mathbb{R}_{\ge 0}^{|E| \times |V|}$ (here, $\M{X}$ does not necessarily refer to a 0,1 incidence matrix) and $\M{S} \in \mathbb{R}_{\ge 0}^{|V| \times |V|}$ and $\{\beta,\gamma\} \ge 0$ are some weighting parameters. 
The matrix $\M{M}$ is then used to obtain clusters similar to the standard NMF clustering procedure. 
Note that if $ \beta = \gamma = 0$ then a standard NMF objective is recovered.
For our experiments we set $\gamma$ and $\beta$ as recommended by Du et al.
For this data set $\M{X}$ is set to  the EDVW incidence matrix $\M{R}$ and $\M{S}$ is a symmetric, $0,1$ matrix indicating if two patents cite each other.
Additionally, we adapt J-NMF to utilize Chung's Laplacian via the matrix $\M{T}$, eq. (\ref{eq:laplacian_core}), and decompose two symmetric matrices.
We refer to this as Joint-Symmetric NMF (JS-NMF), the new objective is given in Eqn. (\ref{eq:JSNMF}).
\begin{equation}
    \label{eq:JSNMF}
    \min_{ \{\M{M},\M{\hat{M}},\M{\Tilde{M}} \} \ge 0} \| \M{C} - \M{M}\M{\hat{M}}^\Tra \|_F^2 +\alpha \|\M{M} - \M{\hat{M}}\|_F^2 + \gamma\| \M{S} - \M{M} \M{\Tilde{M}}^\Tra \\|_F^2 + \beta \|\M{M} - \M{\Tilde{M}}\|_F^2
\end{equation}
where $\M{C} \in \mathbb{R}_{\ge 0}^{|V| \times |V|}$ is $\M{T}$ and $\M{S}$ is the patent citation data as in Eqn. (\ref{eq:JNMF}).
We note that JS-NMF performs the best out of all algorithms for 3 of the 4 US-Patent groups, as can be seen in Figure \ref{USP_plts}.
Due to the enforcement of symmetry for both norms in Eqn. (\ref{eq:JSNMF}), there is an additional parameter $\alpha$, similar to $\gamma$ and $\beta$, that must be chosen.
For fairness we attempt to generalize Du et al.'s recommendation for parameter setting but believe performance could further be improved if more effort were put into selecting these parameters.
Results are visualized in Figure \ref{USP_plts}.

\subsubsection{RCV1}
\label{sec:RCV1}

\begin{figure}[tb]
\centering
\includegraphics[width=0.48\textwidth]{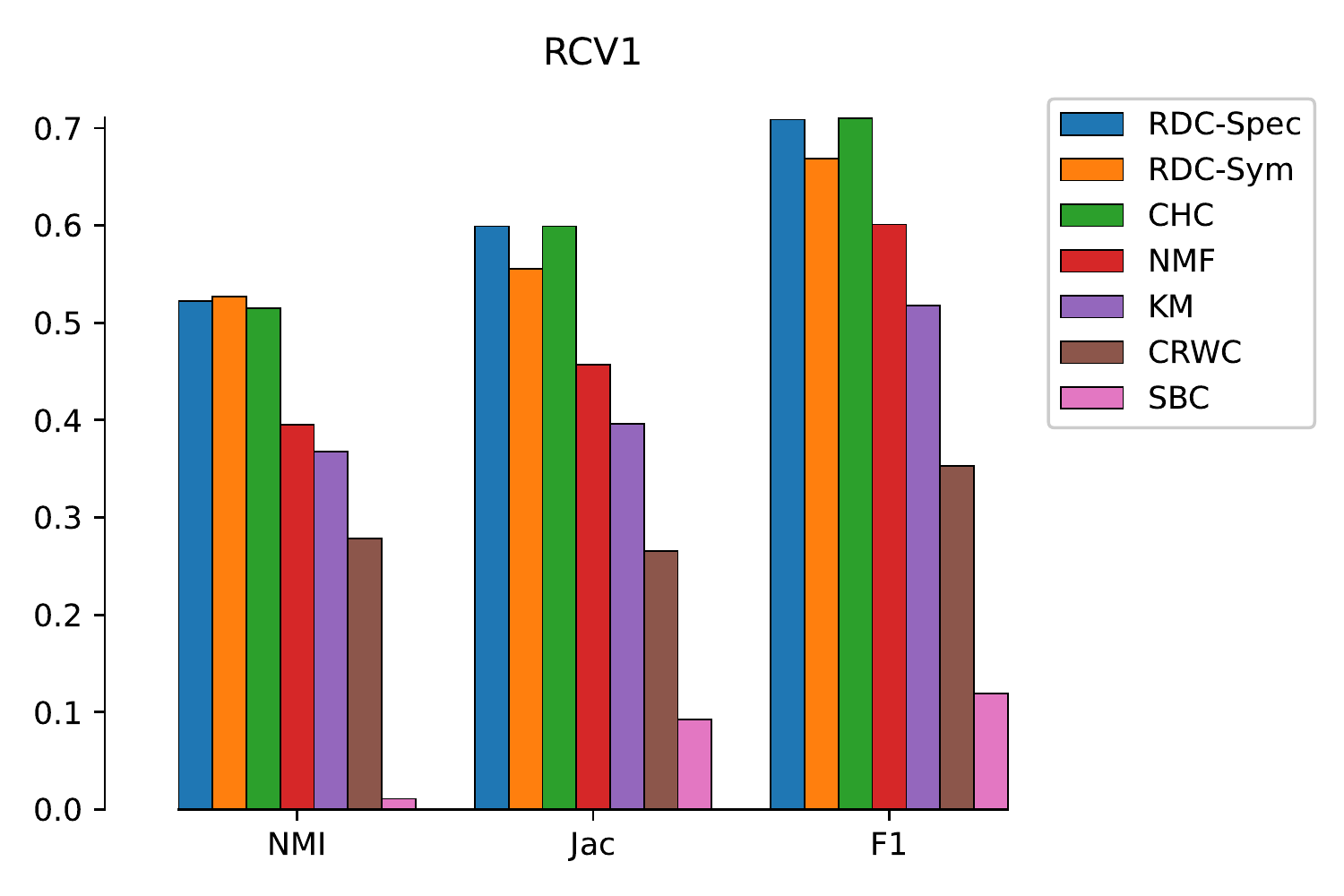}
\caption{Clustering scores based on the measures of NMI, Jaccard index, and F1 score, for 
RCV1 Data.}
\label{rcv1_plts}
\end{figure}

The Reuters Corpus Volume 1 (RCV1)\footnote{http://www.daviddlewis.com/resources/testcollections/rcv1/} data set is a collection of newswire stories.
Table 2 lists some basic statistics for this data. 
This is the largest data set we run on and we observe a similar trend to the two previous experiments.
RDC-Spec and RDC-Sym perform competitively with CHC, with scores separated by very slim margins. 
The scores for each clustering quality metric are visualized in Figure \ref{rcv1_plts}.


\subsubsection{Gene-Disease Data}
\label{sec:Gene}
This dataset consists of collections of genes associated with human diseases, taken from DisGeNET\footnote{https://www.disgenet.org/}. 
This may be modeled as a hypergraph in which vertices are diseases, and genes are hyperedges. Table \ref{text_data_stats_table} presents basic statistics for this data. 
For each disease-gene pair, DisGeNET computes a ``Gene Disease Association" (GDA) score between 0 and 1, which is based on the number of and types of sources supporting that disease-gene association. 
Higher values indicate stronger associations; see DisGeNET's documentation\footnote{https://www.disgenet.org/dbinfo} for full details. 
For our purposes, GDA scores serve naturally as EDVW for the disease-gene hypergraph.

Since this dataset lacks ground truth clusters, we turn to other clustering quality metrics.
Du et al. in \cite{HSymNMF2} report the average-Ncut value obtained by various clustering algorithms on a number of data sets.
This technique is difficult to apply directly to our work where different algorithms are based on different representations of a hypergraph: different representations assign different edge weights which affect the Ncut values.
To this end we run RDC-Spec, RDC-Sym, and CHC on the Gene-Disease hypergraph and compute the average-Ncut and average-conductance, Eqns. (\ref{av_normalized_cut}) and (\ref{av_conductance}), of each clustering over 10 different sparsifications of the hypergraph.
These 3 algorithms are chosen as they tend to achieve consistently high scores on the previous data sets.
Each sparsified input data set is clustered 10 times and the average of the average-normalized cut and average-conductance is computed.
We report the best scores for each algorithm.
\begin{align}
    \text{Av-Ncut}(S_1, \cdots , S_k) &= \frac{1}{2k} \sum_{i=1}^k \frac{\text{vol}{(\partial S_i)}}{\text{vol}(S_i)} \label{av_normalized_cut} \\
    \text{Av-Cond}(S_1, \cdots , S_k) &= \frac{1}{2k} \sum_{i=1}^k \frac{\text{vol}{(\partial S_i)}}{\min\{\text{vol}(S_i),\text{vol}(S_i^c)\}} \label{av_conductance}
\end{align}

Additionally, for every pair of algorithms that are based on two different hypergraph representations, each algorithm is run on the other algorithm's hypergraph representation, producing two different scores.
For example, the clustering yielded by RDC-Spec operating on $\M{T}$, Eqn. (\ref{eq:laplacian_core}), is taken and its average-Ncut and average-conductance values are computed on the matrix $\M{\Delta}$, Eqn. (\ref{Zhou_lap}).
The same is done for RDC-Sym and the reverse is done for CHC.
Note these matrices have the same non-zero pattern.
Therefore differences in cuts values and clusterings are due to  edge-weightings.
Tables \ref{ncut_table15} and \ref{ncut_table25} present the results for 15 and 25 clusters respectively.
These cluster sizes were chosen arbitrarily.
RDC-Spec and RDC-Sym achieve lower cut and conductance scores than CHC on both graphs.

\begin{table}
\begin{tabular}{c|c|c|c}%
\bfseries Graph & \bfseries RDC-Spec & \bfseries RDC-Sym & \bfseries CHC
\begin{filecontents*}{CSV_Data/GD_cuts_k15.csv}
Graph,RDC-Spec,RDC-Sym,CHC
ANC-$\M{T}$,0.1470,0.1177,0.1953
ANC-$\M{\Delta}$,0.2541,0.2218,0.3019
ACo-$\M{T}$,0.2999,0.2479,0.3906
ACo-$\M{\Delta}$,0.5193,0.4654,0.6038
\end{filecontents*}
\csvreader[head to column names]{CSV_Data/GD_cuts_k15.csv}{}
{\\\hline\csvcoli&\csvcolii&\csvcoliii&\csvcoliv}
\end{tabular}
\caption{Average-Normalized cut (ANC) and Average-Conductance (ACo) values for RDC-Spec, RDC-Sym, and CHC on matrices $\M{T}$ and $\M{\Delta}$ with 15 clusters}
\label{ncut_table15}
\end{table}

\begin{table}
\begin{tabular}{c|c|c|c}%
\bfseries Graph & \bfseries RDC-Spec & \bfseries RDC-Sym & \bfseries CHC
\begin{filecontents*}{CSV_Data/GD_cuts_k25.csv}
Graph,RDC-Spec,RDC-Sym,CHC
ANC-$\M{T}$,0.1681,0.1331,0.2059
ANC-$\M{\Delta}$,0.2879,0.2492,0.3133
ACo-$\M{T}$,0.1681,0.1363,0.2059
ACo-$\M{\Delta}$,0.2879,0.2558,0.3133
\end{filecontents*}
\csvreader[head to column names]{CSV_Data/GD_cuts_k25.csv}{}
{\\\hline\csvcoli&\csvcolii&\csvcoliii&\csvcoliv}
\end{tabular}
\caption{Average-Normalized cut (ANC) and Average-Conductance (ACo) values for RDC-Spec, RDC-Sym, and CHC on matrices $\M{T}$ and $\M{\Delta}$ with 25 clusters.}
\label{ncut_table25}
\end{table}

\section{Conclusion}
We presented a flexible framework for clustering a hypergraph and showed that edge-dependent vertex weights represents the information in a hypergraph well in the context of clustering. 
As recently proposed 
\cite {DBLP:conf/icml/ChitraR19}, these weights may be utilized to define hypergraph random walks, which naturally yield a number of different hypergraph Laplacians via the representative digraph of the random walk. 
Focusing on Chung's normalized digraph Laplacian, we explained its effectiveness as a representation of a hypergraph for clustering based on its relationship to a normalized cut criterion, and proposed a suite of clustering algorithms that utilize this input. 
We demonstrated the viability of our frameworks in comparison to other methods through experiments on 3 text datasets with ground truth, and on a gene-disease relation data set via well-known partition quality metrics. 
We found algorithms utilizing the proposed hypergraph Laplacian performed consistently well and frequently better than other methods. 

Many directions remain for future work. 
First, among the hypergraph Laplacians considered in Section \ref{subsec:laps}, we only utilized Chung's normalized digraph Laplacian in our experiments, although other Laplacians can be derived from the representative digraph. 
It would be interesting to explore whether these Laplacians are effective representations for hypergraph clustering -- particularly Li and Zhang's asymmetric digraph Laplacian \cite{li2010random}, which we studied in Proposition \ref{prop:liZhang}, as well as complex-valued digraph matrices utilized in recent clustering work \cite{Cucuringu2019}. 
Additionally, we observed better performance from joint methods that utilize multiple representations and many other combinations here can be explored \cite{Du_JNMF}.

Second, rather than utilize EDVW random walks to form Laplacians for clustering, an alternative approach is to use measures associated with the random walk itself, such as hitting and commute times. 
Such parameters serve as relational measures between pairs of vertices; for instance, hitting times measure the expected number of steps until one vertex is reached from another, and have been utilized in digraph clustering schemes \cite{chen2008clustering}. 
In the context of EDVW hypergraph random walks, such metrics may be of similar use for hypergraph clustering.

\begin{acks}
This work was partially funded under the High Performance Data Analytics (HPDA) program
at the Department of Energy's Pacific Northwest National Laboratory. PNNL Information Release:
PNNL-SA-153213. Pacific Northwest National Laboratory is operated by Battelle Memorial Institute
under Contract DE-ACO6-76RL01830.
Koby Hayashi acknowledges support from the United States Department of Energy through the Computational Sciences Graduate Fellowship (DOE CSGF) under grant number: DE-SC0020347.
\end{acks}

\bibliographystyle{ACM-Reference-Format}
\bibliography{sample-base}

\appendix




\end{document}